\documentclass[10pt,journal,compsoc]{IEEEtran} %\documentclass[twoside,11pt]{article}

\usepackage{amssymb}
\usepackage{amsmath}
\usepackage{graphicx}
\usepackage{subfigure}

\makeatletter

\newcommand{\be}{\begin{equation}}
\newcommand{\ee}{\end{equation}}
\newcommand{\prob}[1]{Pr \left\{#1\right\}}
\newcommand{\mean}[2]{\mathbb{E}_{#1} \left\{ #2 \right\}}
\newcommand{\im}{C} %{\hat{\mathcal{M}}}
\newcommand{\am}{T} % {\mathcal{M}}
\newcommand{\sx}{S_{X}}
\newcommand{\imx}{\im(x)}
\newcommand{\amx}{\am(x)}
\newcommand{\ix}{X_{C}}
\newcommand{\ax}{X'}
\newcommand{\set}[1]{\left\{#1\right\}}

\newcommand{\imhx}{H(x)}
\newcommand{\ixh}{X_{H}}

\newcommand{\sxp}{S'_{X}}
\newcommand*{\rom}[1]{\expandafter\@slowromancap\romannumeral #1@}

\newcommand{\unif}{\mathcal{U}} %{\mbox{\footnotesize{unif}}}

\newtheorem{theorem}{Theorem}

\begin{document}

\title{Validation of Matching}

\author{Ya~Le,
        Eric~Bax,
        Nicola~Barbieri,
        David~Garcia~Soriano,
        Jitesh~Mehta,
        and~James~Li% <-this % stops a space
\IEEEcompsocitemizethanks{
\IEEEcompsocthanksitem Ya Le: yle@stanford.edu, Stanford University, Dept. of Statistics\protect\\
\IEEEcompsocthanksitem Eric Bax (baxhome@yahoo.com), Nicola Barbieri, David Garcia Soriano, and Jitesh Mehta: Yahoo\protect\\
\IEEEcompsocthanksitem James Li: Facebook.}% <-this % stops a space
}

\IEEEtitleabstractindextext{%
\begin{abstract}
We introduce a technique to compute probably approximately correct (PAC) bounds on precision and recall for matching algorithms. The bounds require some verified matches, but those matches may be used to develop the algorithms. The bounds can be applied to network reconciliation or entity resolution algorithms, which identify nodes in different networks or values in a data set that correspond to the same entity. For network reconciliation, the bounds do not require knowledge of the network generation process. 
\end{abstract}}

\maketitle

%\begin{keywords}
%validation, PAC bounds, matching, network reconciliation, entity resolution
%\end{keywords}

\section{Introduction}
\IEEEPARstart{O}{ur} interest in validation of matching stems from work on large-scale network reconciliation. Network reconciliation is the process of matching nodes in different networks that correspond to the same entity \cite{bhattacharya07}, \cite{tian08}, \cite{zaslavskiy08}. Network reconciliation can combine information from multiple networks to form a combined network with richer information about the entities, represented by nodes, and their interactions with each other, represented by edges. For example, in social networks, reconciling profiles across networks can produce richer information about people and their relationships. This can be a nontrivial challenge for internet-scale data, because, for example, there are over 35,000 people named John Smith on LinkedIn. In biology, matching proteins across species based on their webs of interactions with other proteins can help us induce which proteins play similar roles in different organisms. By examining which structures are preserved across different proteins that have similar roles, we can learn which structures are fundamental to the function of those proteins. 

Some network reconciliation algorithms are iterative. They establish some initial ``seed matches" based on node data or by matching the highest-degree nodes in the different networks to each other \cite{tian08,korula14}. Then they use shared connections to establish more matches, based on the idea that if more of the connections to a node in one network match connections to a node in another network, then the nodes are more likely to match. With each iteration, more matches may be established, offering opportunities for further matches. This iterative matching process is called percolation. Refer to \cite{yartseva13} for some theory on how many seeds are needed for percolation to succeed. 

Some network reconciliation algorithms do not require a seed set. Nunes et. al. \cite{nunes12} reduce the problem to a binary classification task, where the goal is to classify pairs of nodes as either matched or not matched. Zaslavskiy et. al. \cite{zaslavskiy08} propose a convex-concave programming approach by rewriting the problem as a least squares problem on the set of permutation matrices. Pedarsani et. al. \cite{pedarsani13} cast the problem in a Bayesian framework based on a definition of the probability of correctly mapping two nodes. Refer to \cite{foggia14} for a review of network reconciliation algorithms.

Our matching problem setting is similar to the transductive setting for classification, from Vapnik \cite{vapnik98}, where there is a set of training examples with known inputs and class labels and a set of working examples with known inputs and unknown class labels, and the goal is to use the available training and working data to develop a classifier that classifies the working examples with a low error rate. For results on validation of network classifiers (rather than reconciliation algorithms) in transductive settings, refer to \cite{li12} and \cite{bax13}. For theory and insight on why collective classification succeeds in general settings and validation methods for it, refer to \cite{london14}. For network reconciliation, we assume that we know some network data, consisting of some node data and the links, for both networks involved in the matching, and our goal is to use that network data to match nodes as accurately as possible between the networks. 

This paper presents a technique to compute probably approximately correct (PAC) bounds on the precision and recall of matching algorithms. When applied to network reconciliation algorithms, the technique does not require knowledge of how the networks were generated or rely on any specific model of network generation. The technique requires some verified matches. However, it produces valid bounds even if the verified matches are used to develop the matching algorithm, by taking advantage of a withhold-and-gap (WAG) \cite{bax15} strategy -- first withhold some data, train a holdout matching algorithm, validate its precision and recall on the holdout data, then train a matching algorithm on all data, and use unlabeled data to evaluate the rate of differences in matches between the two algorithms. The precision and recall for the matching algorithm trained on all data cannot be worse than the precision and recall of the holdout matching algorithm plus the rates of differences between the two algorithms, because in the worst case all differences are errors by the matching algorithm trained on all data. 

The validation of precision and recall for matching algorithms is similar to the validation of classifiers over joint distributions of inputs and outputs. Both use empirical means over samples to construct bounds. However, in our setting, the sets of items to match are known, and we use empirical means over samples drawn without replacement from a finite population to bound means over the finite population. In most validation settings, samples from a joint input-output distribution are drawn with replacement and used to bound means over the input-output distribution. So, as the basis for our bounds on precision and recall for matching, we need slightly different empirical-mean bounds than those typically used in machine learning. 

This paper is organized as follows. Section \ref{sec:bounds} reviews the empirical-mean bounds that we need to validate matching in our setting. Section \ref{sec:holdout} presents techniques to validate precision and recall for holdout matching algorithms. Section \ref{sec:complete} presents techniques to extend the bounds on precision and recall to matching algorithms developed using all available data. Section \ref{sec:extensions} extends the bounds on precision and recall to other types of error bounds for matching and to some applications of matching beyond network reconciliation. Section \ref{sec:sampling} explains how to sample validation data from a finite population in our transductive setting. Section \ref{sec:numbers} offers some numerical results. Section \ref{sec:conclusion} concludes with some directions for future research.

\section{Bounds for Population Means} \label{sec:bounds}

The validation methods in later sections use means over samples drawn from a finite population to estimate and bound population means. The mean bounds most commonly used in machine learning, such as \cite{hoeffding63}, \cite{boucheron13}, \cite{audibert04}, and \cite{maurer09}, use means over samples from a distribution to bound means over the distribution. There are similar methods for population means; we review them here.

Let $X$ be a size-$n$ finite population, and let $S$ be a size-$s$ random sample drawn uniformly at random without replacement from $X$. Let $f$ denote a real-valued function defined on $X$, i.e., $f:X\rightarrow\mathbb{R}$. Assume that $f$ is bounded: there exist $a$ and $b$ such that
$a\leq f(x)\leq b$, $\forall x \in X$.

The goal is to estimate the population mean:
\be
\mu = \mean{x \sim \unif(X)}{f(x)} = \frac{1}{n} \sum_{x \in X}f(x)
\ee
and to compute PAC bounds for it, using the sample $S$. (We use $\unif()$ to denote a uniform distribution.) The sample mean:
\be
\hat{\mu} = \mean{x \sim \unif(S)}{f(x)} = \frac{1}{s} \sum_{x \in S}f(x)
\ee
is an unbiased estimator for $\mu$. Let $p^{+}(X,S,f,a,b,\delta)$ and $p^{-}(X,S,f,a,b,\delta)$ denote PAC upper and lower bounds on $\mu$, with bound failure probability at most $\delta$:
\begin{eqnarray}
Pr\left\{\mu>p^{+}(X,S,f,a,b,\delta)\right\} & \leq & \delta,\\
Pr\left\{\mu<p^{-}(X,S,f,a,b,\delta)\right\} & \leq & \delta.
\end{eqnarray}
We will outline some methods to compute these bounds.

\subsection{Hoeffding's Inequality}

Hoeffding \cite{hoeffding63} and Chv\'atal \cite{chvatal79} show that the well-known Hoeffding inequality for distribution means also applies to population means:

\begin{eqnarray}
p^{+}(X,S,f,a,b,\delta) & = & \hat{\mu}+(b-a)\sqrt{\frac{\ln1/\delta}{2s}},\\
p^{-}(X,S,f,a,b,\delta) & = & \hat{\mu}-(b-a)\sqrt{\frac{\ln1/\delta}{2s}}.
\end{eqnarray}

Hoeffding and Chv\'atal also offer tighter bounds if the variance of $f$ is known to be small. (The bounds above are based on a worst-case assumption about the variance of $f$.)

\subsection{Empirical Bernstein-Serfling Inequality}
Audibert \cite{audibert04} developed, and Maurer and Pontil \cite{maurer09} improved, empirical Bernstein bounds, named for a bound by Bernstein \cite{bernstein37}. Empirical Bernstein bounds use the variance of $f$ over the sample to allow the use of tighter versions of traditional bounds if the standard deviation of $f$ is small compared to its range, which is common for error rates, since error rates tend to be small in cases of interest. Bardenet and Maillard \cite{bardenet13} have developed similar bounds, called empirical Bernstein-Serfling bounds, for population means:
\be
p^{+}(X,S,f,a,b,\delta)
\ee
\be
= \hat{\mu}+\hat{\sigma}_{s}\sqrt{\frac{2\rho_{s}\log(5/\delta)}{s}}+\frac{\kappa(b-a)\log(5/\delta)}{s}
\ee
and
\be
p^{-}(X,S,f,a,b,\delta)
\ee
\be
= \hat{\mu}-\hat{\sigma}_{s}\sqrt{\frac{2\rho_{s}\log(5/\delta)}{s}}-\frac{\kappa(b-a)\log(5/\delta)}{s},
\ee
where 
\begin{equation}
\rho_{s}=\begin{cases}
1-(s-1)/n & \mbox{if }s\leq n/2\\
(1-s/n)(1+1/n) & \mbox{if }s>n/2
\end{cases},
\end{equation}
$\hat{\sigma}_{s}^{2}=\sum_{i,j=1}^{s}(x_{i}-x_{j})^{2}/(2s^{2})$, and $\kappa=\frac{7}{3}+\frac{3}{\sqrt{2}}$.

\subsection{Direct Computation}

If $f$ can only take values 0 or 1, then we can use direct computation to produce tight bounds. For means over distributions, we can use binomial tail inversion, as outlined by Hoel \cite{hoel54} and Langford \cite{langford05}. For population means, we can use hypergeometric tail inversion, from Chv\'atal \cite{chvatal79}, as follows.

If $f$ can only take value 0 or 1, then $s\hat{\mu}=\sum_{i=1}^{s}f(x_{i})$ follows a hypergeometric distribution. Let $m=n \mu = \sum_{x\in X}f(x)$. Then 
\begin{equation}
Pr\left\{s\hat{\mu}=k\right\}={m \choose k}{n-m \choose s-k}\Big/{n \choose s}.
\end{equation}
The hypergeometric tail distribution is 

\be
H^+(m,n,s,k)  =  \sum_{j=k}^{s}Pr\left\{s\hat{\mu}=j\right\}
\ee
\be
= \sum_{j=k}^{s}{m \choose j}{n-m \choose s-j}\Big/{n \choose s}
\ee
and
\be
H^-(m,n,s,k)  =  \sum_{j=0}^{k}P\left\{s\hat{\mu}=j\right\}
\ee
\be
= \sum_{j=0}^{k}{m \choose j}{n-m \choose s-j}\Big/{n \choose s}.
\ee
So we have PAC bounds:
\be
p^{+}(X,S,f,a,b,\delta)  =  \max \left\{p:H^-(np,n,s,s\hat{\mu})\geq\delta\right\}
\ee
and
\be
p^{-}(X,S,f,a,b,\delta)  =  \min \left\{p:H^+(np,n,s,s\hat{\mu})\geq\delta\right\},
\ee
which are tight bounds for $\mu$, except for any rounding or approximation errors from the computation. Use either Loader's method \cite{loader00} or Stirling's approximation to compute $H^-$ and $H^+$.

\section{Validation for Holdout Algorithms} \label{sec:holdout} 
In this section we present PAC bounds on precision and recall for holdout algorithms, which are developed while withholding a set of validation data. Let $X$ and $Y$ be the sets for which we match elements. (For example, $X$ may be the set of nodes in one network, and $Y$ may be the nodes in another.) The validation data is a sample from $X$ for which all true matches are known and algorithm matches are computed. (For example, in matching nodes between social networks, the validation data is a set of nodes in one network for which any true matches in the other network are known.) In the next section, we extend the bounds to algorithms developed using all available data. 

Let $\am$ be the set of all true matches in $X \times Y$. For $x \in X$, let $\amx$ be the set of actual matches that include $x$. Let $\imhx$ be the set of $x$-$y$ pairs that the holdout algorithm computes as matches that include $x$. Let $(x,y)$ be a pair selected uniformly at random from $X \times Y$.

Define single-node precision:
\be
p_H(x) = \frac{\imhx \cap \amx}{\imhx}
\ee
\be
= \prob{(x,y) \in \amx | (x,y) \in \imhx}.
\ee
If the algorithm computes one match for $x$ and there is at most one true match, then $p_H(x)$ is one if the algorithm match is a true match and zero otherwise. In general, $p_H(x)$ is the fraction of algorithm matches for $x$ that are true matches. If the algorithm computes no matches for $x$, then $p_H(x)$ is undefined. So let 
\be
\ixh = \set{x \in X | \imhx \not= \emptyset},
\ee
and define holdout precision:
\be
P_H = \mean{x \sim \unif(\ixh)}{p_H(x)}.
\ee
If the algorithm computes a match for each $x$ and there is a true match for each $x$, then $P_H$ is the precision as usually defined in information-retrieval settings. 

Similarly, define single-node recall: 
\be
r_H(x) = \frac{\imhx \cap \amx}{\amx}
\ee
\be
= \prob{(x,y) \in \imhx | (x,y) \in \amx}.
\ee
If there are no actual matches for $x$, then $r_H(x)$ is undefined. So let 
\be
\ax = \set{x \in X | \amx \not= \emptyset},
\ee
and define holdout recall to be
\be
R_H = \mean{x \sim \unif(\ax)}{r_H(x)}.
\ee

Let $\sx$ be a sample drawn from $X$, uniformly at random without replacement. Suppose we know all true matches for each $x \in \sx$ and we compute their algorithm matches. Then we can use the sample mean for $p_H(x)$ over $x \in \sx \cap \ixh$ and for $r_H(x)$ over $s \in \sx \cap \ax$ to bound holdout precision and recall, respectively. 

\begin{theorem}[Holdout Precision and Recall] \label{thm:query_holdout}
For any bound failure probability $\delta > 0$, with probability at least $1 - \delta$:
\be
P_H \geq p^-(\ixh, \sx \cap \ixh, p_H, 0, 1, \delta).
\ee
Also, with probability $1 - \delta$: 
\be
R_H \geq p^-(\ax, \sx \cap \ax, r_H, 0, 1, \delta).
\ee
\end{theorem}

\begin{IEEEproof}
Since the elements of $\sx$ are selected uniformly at random without replacement from $X$, $\sx \cap \ixh$ has the same distribution as selecting a size-$|\sx \cap \ixh|$ sample from $\ixh$ uniformly at random without replacement. Similarly, $\sx \cap \ax$ has the same distribution as selecting a size-$|\sx \cap \ax|$ sample from $\ax$ uniformly at random without replacement. 
\end{IEEEproof}

Theorem \ref{thm:query_holdout} is valid if we continue to draw samples from $\sx$ until $\sx \cap \ixh$ or $\sx \cap \ax$ reaches a desired sample size. If there is a method to sample directly from $\ixh$ or $\ax$, then those samples also produce valid bounds. 

\section{Validation for Complete Algorithms} \label{sec:complete}

In this section, we bound precision and recall for matching algorithms developed using all available data, which we call complete algorithms. To do this, we bound the rate of disagreement between a holdout algorithm and the complete algorithm. In the worst case, every disagreement is a mismatch by the complete algorithm. So the mismatch rate for the complete algorithm is no greater than the rate for the holdout algorithm plus the rate of disagreement. 

Let $\im$ be the set of identified matches returned by the complete classfier, which can be developed using all available data, including the validation data. Let $\sxp$ be a sample independent of $\sx$ and drawn uniformly at random without replacement from $X$. Assume that $\imx$ and $\imhx$ are computed for $x\in\sx$ and $x\in\sxp$. Define $p(x)$, $\ix$, $P$, $r(x)$, and  $R$ as complete versions of the holdout-oriented definitions in Section \ref{sec:holdout}:
\begin{eqnarray}
p(x) & = & \frac{\imx \cap \amx}{\imx}, \\
\ix & = & \set{x\in X | \imx \neq \emptyset}, \\
P & = & \mean{x \sim \unif(\ix)}{p(x)}, \\
r(x) & = & \frac{\imx \cap \amx}{\amx}, \\
R & = & \mean{x \sim \unif(\ax)}{r(x)}.
\end{eqnarray}

Here are two remarks before presenting the results:
\begin{enumerate}
	\item We use an extra sample, $\sxp$, to validate rates of disagreement between the holdout and complete algorithms. For this sample, we only need to know the holdout algorithm matches and the complete algorithm matches -- not the true matches. Therefore, collecting data for $\sxp$ is usually cheaper than collecting validation data for $\sx$, which includes true matches. So we can usually make the size of sample $\sxp$ very large, producing a tight bound on the rate of disagreement between the holdout and complete algorithms.
	\item For simplicity of notation, we use the same bound failure probability $\delta$ for each term in our results. But, in practice, we can use different bound failure probabilities for different terms.
\end{enumerate}

\begin{theorem}[Complete Recall] \label{thm:complete_recall}
	For any bound failure probability $\delta>0$, with
	probability at least $1-3\delta$, 
	\begin{equation}
	R \geq p^{-}\left(\ax,\sx\cap\ax,r_H,0,1,\delta\right)
	\ee
	\be
	-\frac{p^{+}\left(X,\sxp,d_{r},0,1,\delta\right)}{p^{-}\left(X,\sx,\mathbf{1}_{\ax},0,1,\delta\right)},
	\end{equation}
	where $d_{r}(x)=\mathbf{1}_{\{\imhx-\imx\neq\emptyset\}}$.
	\end{theorem}
\begin{IEEEproof}
	Since $\sx$ is a sample selected uniformly
	at random without replacement from $X$, by Theorem \ref{thm:query_holdout}, the recall of the holdout algorithm satisfies
	\begin{equation}
	\prob{R_H \geq p^{-}\left(\ax,\sx\cap\ax,r_{H},0,1,\delta\right)}\geq1-\delta.\label{eq:query_recall1}
	\end{equation}

	The difference in recall between the complete and holdout algorithms is
	\begin{eqnarray}
	&  & R - R_H\\
	& = & \frac{1}{|\ax|}\sum_{x\in\ax}\frac{|\imx\cap\amx|-|\imhx\cap\amx|}{|\amx|}\\
	& \geq & -\frac{1}{|\ax|}\sum_{x\in\ax}\frac{|(\imhx-\imx)\cap\amx|}{|\amx|}\\
	& \geq & -\frac{1}{|\ax|}\sum_{x\in\ax}\mathbf{1}_{\{\imhx-\imx\neq\emptyset\}}\\
	& \geq & -\frac{|X|}{|\ax|}\frac{1}{|X|}\sum_{x\in X}d_{r}(x).
	\end{eqnarray}

	Since $\sxp$ is independent of $\sx$, $\sxp$ is also a sample
	drawn uniformly at random without replacement from $X$
	conditional on $\sx$. Note that $\imx$ is a fixed function conditional
	on $\sx$ since we assume that the training samples of the holdout algorithm are fixed, therefore we can use $\sxp$ to validate $\frac{1}{|X|}\sum_{x\in X}d_{r}(x)$
	conditional on $\sx$:
	\begin{equation}
	\prob{\frac{1}{|X|}\sum_{x\in X}d_{r}(x)\leq p^{+}\left(X,\sxp,d_{r},0,1,\delta\right)\vert\sx}>1-\delta.
	\end{equation}
	We derive the unconditional bound by integrating out $\sx$:
	
	$ $
	\begin{equation}
	\prob{\frac{1}{|X|}\sum_{x\in X}d_{r}(x)\leq p^{+}\left(X,\sxp,d_{r},0,1,\delta\right)}>1-\delta.\label{eq:query_recall2}
	\end{equation}

	Since $\frac{|\ax|}{|X|}=\frac{1}{|X|}\sum_{x\in X}\mathbf{1}_{\ax}(x)$,
	we can apply the bounds from Section \ref{sec:bounds}:
	\begin{equation}
	\prob{\frac{|\ax|}{|X|}\geq p^{-}\left(X,\sx,\mathbf{1}_{\ax},0,1,\delta\right)}>1-\delta.\label{eq:query_recall3}
	\end{equation}

	Combining (\ref{eq:query_recall1}), (\ref{eq:query_recall2}), and
	(\ref{eq:query_recall3}) by the union bound we get the result stated
	in the theorem. \end{IEEEproof}
	
\begin{theorem}[Complete Precision] \label{thm:complete_precision}
	For any bound failure probability $\delta>0$, with
	probability at least $1-4\delta$,
	\be
	P \geq
	\ee
	\be
	\left[p^{-}(X,\sxp,\mathbf{1}_{\ixh},0,1,\delta)p^{-}(\ixh,\sx\cap\ixh,p_{H},0,1,\delta)\right.
	\ee
	\be
	\left. -p^{+}\left(X,\sxp,d_{p},0,3,\delta\right)\right]
	\ee
	\be / p^{+}(X,\sxp,\mathbf{1}_{\ix},0,1,\delta),
	\ee
	where 
	\begin{equation}
	d_{p}(x)=\mathbf{1}_{\ix\cap\ixh}(x)\mathbf{1}_{\{\imx\neq\imhx\}}
	\ee
	\be
	\left(1+\frac{|\imhx-\imx|}{|\imx|}\right)+\mathbf{1}_{\ixh-\ix}(x).
	\end{equation}
\end{theorem}
\begin{IEEEproof}
	Note that 
	\begin{eqnarray}
	P & = & \frac{1}{|\ix|}\sum_{x\in\ix}\frac{|\imx\cap\amx|}{|\imx|}\\
	& = & \frac{1}{|\ix|}\left(\sum_{x\in\ix}\frac{|\imx\cap\amx|}{|\imx|}-\sum_{x\in\ixh}\frac{|\imhx\cap\amx|}{|\imhx|}\right)\nonumber \\
	&  & +\frac{|\ixh|}{|\ix|}\frac{1}{|\ixh|}\sum_{x\in\ixh}\frac{|\imhx\cap\amx|}{|\imhx|}.
	\end{eqnarray}
	The first term can be decomposed into three parts
	\begin{eqnarray}
	&  & \sum_{x\in\ix}\frac{|\imx\cap\amx|}{|\imx|}-\sum_{x\in\ixh}\frac{|\imhx\cap\amx|}{|\imhx|}\\
	& = & \sum_{x\in\ix\cap\ixh}\left(\frac{|\imx\cap\amx|}{|\imx|}-\frac{|\imhx\cap\amx|}{|\imhx|}\right)\nonumber \\
	&  & +\sum_{x\in\ix-\ixh}\frac{|\imx\cap\amx|}{|\imx|}\nonumber \\
	& & -\sum_{x\in\ixh-\ix}\frac{|\imhx\cap\amx|}{|\imhx|}\\
	& = & \mbox{\rom{1}}+\mbox{\rom{2}}+\mbox{\rom{3}}.
	\end{eqnarray}

	Since $|\ix-\ixh|$ and $|\ixh-\ix|$ are usually
	small, we can simply use 0 and $-|\ixh-\ix|$ as the lower
	bound for $\mbox{\rom{2}}$ and $\mbox{\rom{3}}$. For $\mbox{\rom{1}}$,
	we have 
	\begin{eqnarray}
	\mbox{\rom{1}} & \geq & \sum_{\substack{x\in\ix\cap\ixh\\
			\imx\neq\imhx
		}
	}\bigg(\frac{|\imhx\cap\amx|}{|\imx|}-\frac{|\imhx\cap\amx|}{|\imhx|}\nonumber \\
	&  & -\frac{|(\imhx-\imx)\cap\amx|}{|\imx|}\bigg)\\
	& \geq & \sum_{\substack{x\in\ix\cap\ixh\\
			\imx\neq\imhx
		}
	}\left(-1-\frac{|\imhx-\imx|}{|\imx|}\right).
	\end{eqnarray}

	Therefore, 
	\begin{eqnarray}
	&  & P \nonumber \\
	& \geq & -\frac{1}{|\ix|}\nonumber \\
	& & \left[\sum_{\substack{x\in\ix\cap\ixh\\
			\imx\neq\imhx
		}
	}\left(1+\frac{|\imhx-\imx|}{|\imx|}\right)+|\ixh-\ix|\right]\nonumber \\
	&  & +\frac{|\ixh|}{|\ix|}\frac{1}{|\ixh|}\sum_{x\in\ixh}\frac{|\imhx\cap\amx|}{|\imhx|}\\
	& = & \frac{|X|}{|\ix|}\bigg\{-\frac{1}{|X|}\sum_{x\in X}\bigg[\mathbf{1}_{\ix\cap\ixh}(x)\mathbf{1}_{\{\imx\neq\imhx\}}\nonumber \\
	& & \left(1+\frac{|\imhx-\imx|}{|\imx|}\right)+\mathbf{1}_{\ixh-\ix}(x)\bigg] \nonumber \\
	&  & +\frac{|\ixh|}{|X|}\frac{1}{|\ixh|}\sum_{x\in\ixh}\frac{|\imhx\cap\amx|}{|\imhx|}\bigg\}
	\end{eqnarray}

	Note that 
	\begin{equation}
	P_H =\frac{1}{|\ixh|}\sum_{x\in\ixh}\frac{|\imhx\cap\amx|}{|\imhx|}.
	\end{equation}
	By Theorem \ref{thm:query_holdout}, we have a lower
	bound for the precision of the holdout algorithm with probability
	at least 1-$\delta$,
	\begin{equation}
	P_H \geq p^{-}\left(\ixh,\sx\cap\ixh,p_{H},0,1,\delta\right).\label{eq:query_precision1}
	\end{equation}

	Since $\sxp$ is independent of $\sx$, $\sxp$ is also a sample
	selected uniformly at random without replacement from $X$
	conditional on $\sx$. Note that $\imx$ is a fixed function conditional
	on $\sx$, therefore we can use $\sxp$ to validate $\frac{1}{|X|}\sum_{x\in X}d_{p}(x)$
	conditional on $\sx$, where $0\leq d_{p}(x)\leq3$:
	\begin{equation}
	\prob{\frac{1}{|X|}\sum_{x\in X}d_{p}(x)\leq p^{+}\left(X,\sxp,d_{p},0,3,\delta\right)\vert\sx}>1-\delta.
	\end{equation}
	We derive the unconditional bound by integrating out $\sx$:
	
	$ $
	\begin{equation}
	\prob{\frac{1}{|X|}\sum_{x\in X}d_{p}(x)\leq p^{+}\left(X,\sxp,d_{p},0,3,\delta\right)}>1-\delta.\label{eq:query_precision2}
	\end{equation}

	Since 
	\be
	\frac{|\ix|}{|X|}=\frac{1}{|X|}\sum_{x\in X}\mathbf{1}_{\ix}(x)
	\ee
	and 
	\be
	\frac{|\ixh|}{|X|}=\frac{1}{|X|}\sum_{x\in X}\mathbf{1}_{\ixh}(x),
	\ee
	using the bounds from Section \ref{sec:bounds}, we have PAC bounds
	\begin{eqnarray}
	\prob{\frac{|\ix|}{|X|}\leq p^{+}\left(X,\sxp,\mathbf{1}_{\ix},0,1,\delta\right)} & > & 1-\delta,\label{eq:query_precision3}\\
	\prob{\frac{|\ixh|}{|X|}\geq p^{-}\left(X,\sxp,\mathbf{1}_{\ixh},0,1,\delta\right)} & > & 1-\delta.\label{eq:query_precision4}
	\end{eqnarray}

	Combining (\ref{eq:query_precision1}), (\ref{eq:query_precision2}),
	(\ref{eq:query_precision3}) and (\ref{eq:query_precision4}) by the
	union bound we get the result stated in the theorem. \end{IEEEproof}

\section{Extensions} \label{sec:extensions}
This section outlines a few ways to extend the results from the previous sections. One extension provides different confidence levels for matches identified by different algorithms or identified by a single algorithm that provides match scores as well as matches. Other extensions bound matching error rate, instead of precision or recall, and validate matching between data fields, instead of nodes in different networks. 

\subsection{Confidence Levels for Different Types of Matches}
The methods described in previous sections validate a single matching algorithm. In some applications, multiple matching algorithms are applied, producing multiple sets of algorithm matches. Similarly, some network reconciliation algorithms return a score for each match. The score is designed to indicate a level of similarity between the nodes in a matched pair. In this case, the range of scores may be partitioned, and scores in different parts of the range can be placed in different sets of algorithm matches.

Use union bounds to produce simultaneous PAC bounds for precision or recall over the different sets of algorithm matches. For example, if there are two sets of algorithm matches, and we validate the recall of one to be at least 80\% and the other to be at least 90\%, each with bound failure probability at most 2.5\%, then those recall rates both hold with probability at least 100\% - $2(2.5\%)$ = 95\%. Simultaneous bounds support different levels of confidence for different sets of identified matches, which may affect how we choose to handle the matches in applications. For example, in an application to merge contact information based on matching user profiles of contacts across social networks, if we have only 80\% precision for a set of identified matches, then we may choose to ask the user to verify matches from that set. If we have 99\% precision from a different set, we may choose to merge those matches automatically, without asking the user for confirmation. 

\subsection{Validation of Error Rates}
The validation methods for precision and recall can be easily extended to other measures of accuracy for matching algorithms. For example, for complete algorithms, we can define single-node error for node $x \in X$ as
\be
w(x) = \mathbf{1}_{\imx \not= \amx},
\ee
meaning that the algorithm has an error on a node if it either fails to identify all matches for the node or identifies any false matches for it. The error rate for the algorithm can be defined as the average of single-node errors over all nodes in $X$. As we did for precision and recall rates, we can validate the error rate by using a sample of nodes with verified actual matches to bound the error rate for a holdout algorithm and using a sample of nodes to validate the rate of disagreement between the holdout algorithm and the algorithm developed on all available data. 

\subsection{Entity Resolution}
In some applications, we are given a set of data fields -- for example names, phone numbers, email addresses, and locations -- and some information about links between them -- for example that they are in the same entry of a contact list or that they occur together in a document. The goal is to match pairs of data fields that refer to the same entity or to aggregate data fields by entity. This problem is sometimes called entity resolution \cite{bhattacharya07}. We can extend our validation methods to entity resolution algorithms for field matching or aggregation, as follows.

An example of matching pairs of fields is an application that connects an email address to a phone number in case someone wants to reply to an email by phone. For matching pairs of fields, we can set $X = Y$, so that we are matching nodes (fields, in our case) from a set to other nodes in the same set. Do not include pairs consisting of the same node twice in the sets of actual matches $\am$ or algorithm matches $\im$, since we want to validate the matching of different fields to each other. 

An example of field aggregation is an application that extracts fields from free text and structured data, then aggregates those fields for each of a set of people or organizations. For these matching algorithms, we can use $X$ as the set of entities and $Y$ as the set of fields. In general, each $x \in X$ will match to multiple $y \in Y$. In many applications, it is also possible for multiple $x \in X$ to match a single $y \in Y$. For example, multiple people may share a single home phone number and address. 

\section{Sampling Procedure for Validation Over a Finite Population} \label{sec:sampling}
Our validation methods use samples $\sx$, drawn uniformly at random without replacement from $X$. In some applications, we wish to use only a subset of our labeled data as validation data, so we can also use some labeled data to develop the holdout algorithm. If we simply partition a sample into a validation set to be withheld and a training set to develop the holdout algorithm, then the holdout algorithm is not developed independently of the selection of the validation subsample, as required for our bounds to be valid, because the training subsample specifically excludes the validation subsample. Instead, we must allow a chance for some intersection between subsamples so that they are distributed as if the training and validation samples were drawn separately, and each draw was uniformly at random without replacement from $X$. 

Let $L \subseteq X$ be a sample of $X$ drawn uniformly at random without replacement for which we have verified the true matches. Select a validation subsample size $s$ and a training subsample size $t$ such that $s + t = |L|$. Generate holdout development (training) subsample $D$ (with $|D|=t$) and validation subsample $S$ (with $|S|=s$) as follows:
\begin{enumerate}
\item Select $t$ samples uniformly at random without replacement from $L$. These samples are $D$.
\item Generate intersection size $i$ from a hypergeometric distribution with population size $|X|$, number of success states in the population $t$
and number of draws $s$. (So $\prob{i} = \frac{{t \choose i}{{|X| -t} \choose {s-i}}}{{{|X|} \choose s}}$.)
\item Select $i$ samples uniformly at random without replacement from $D$.
\item Select $s-i$ samples uniformly at random without replacement from $L- D$.
\item The validation subsample $S$ consists of the samples from Steps 3 and 4.
\end{enumerate}
Let $D'$ and $S'$ denote two independent samples of size $t$ and $s$, respectively, drawn uniformly at random without replacement from the population $X$. Then the subsamples $D$ and $S$ generated by the above procedure have the same distribution as $D'$ and $S'$. This is because the size of the intersection between $D$ and $S$ has the same distribution as that between $D'$ and $S'$, and the samples drawn in Steps 1, 3, and 4 have the same distribution as those drawn uniformly at random without replacement from the population. See the appendix for a rigorous proof.

\section{Numerical Results} \label{sec:numbers}
This section presents some numerical results for complete recall for typical parameter values for matching users between social networks. This gives a general sense of bound effectiveness. It also allows us to compare the different methods for bounding actual means based on empirical means: Hoeffding, empirical Bernstein, and hypergeometric (direct). Let $|X| = $ 120 million. Assume that the fraction of $x \in X$ that have a match in $Y$ is $\frac{2}{3}$, so $|X'| = $ 80 million. Assume the rate of disagreement between holdout and complete algorithms, $d_r$, is 0.001. Let $|S'_X|$ = 100,000. (These are samples used to assess the rate of disagreement; their actual matches need not be verified.) Let $\delta = \frac{0.05}{3}$, so that we have a 95\% bound success probability. 

\begin{figure}
  \includegraphics[width=\linewidth]{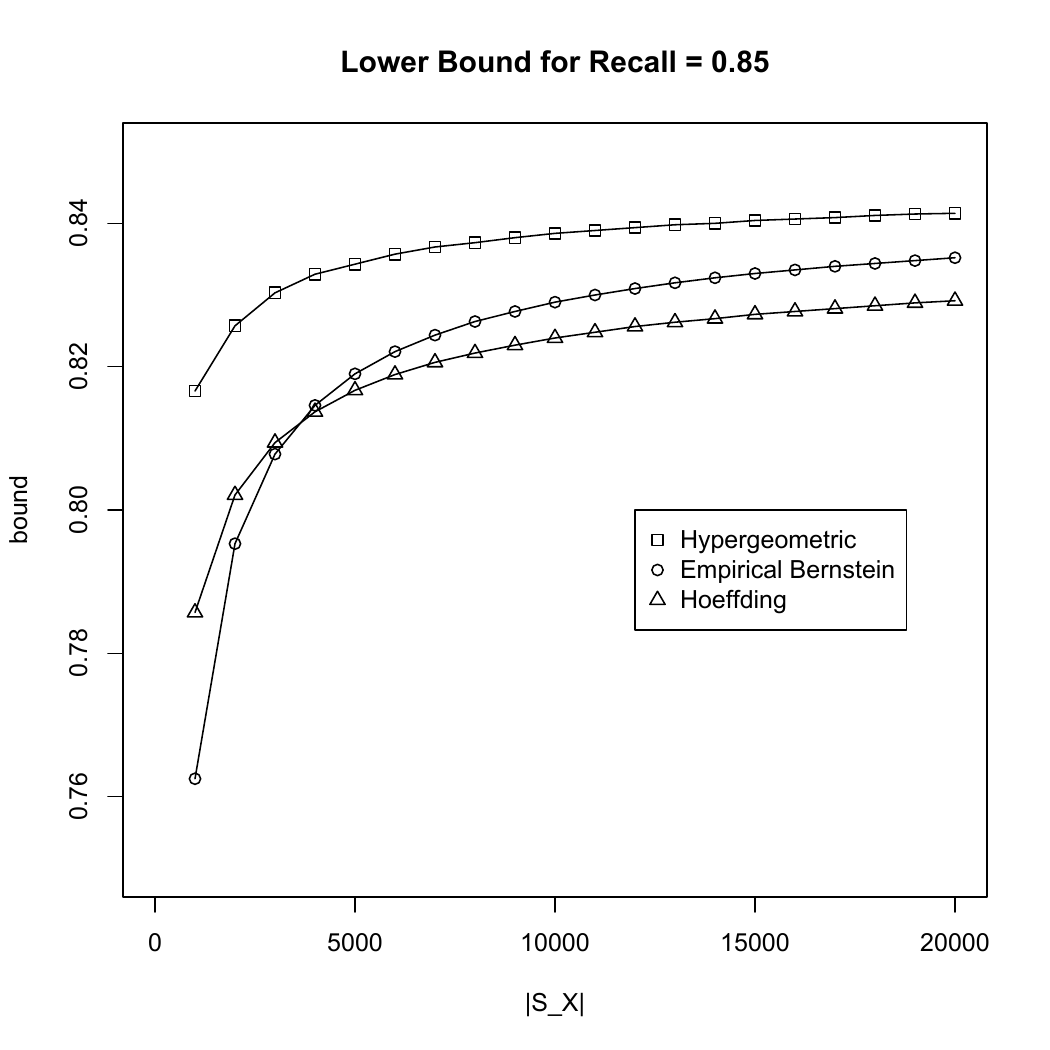}
  \caption{Lower bounds for complete recall using Theorem \ref{thm:complete_recall}.}
  \label{recall_85}
\end{figure}

\begin{figure}
  \includegraphics[width=\linewidth]{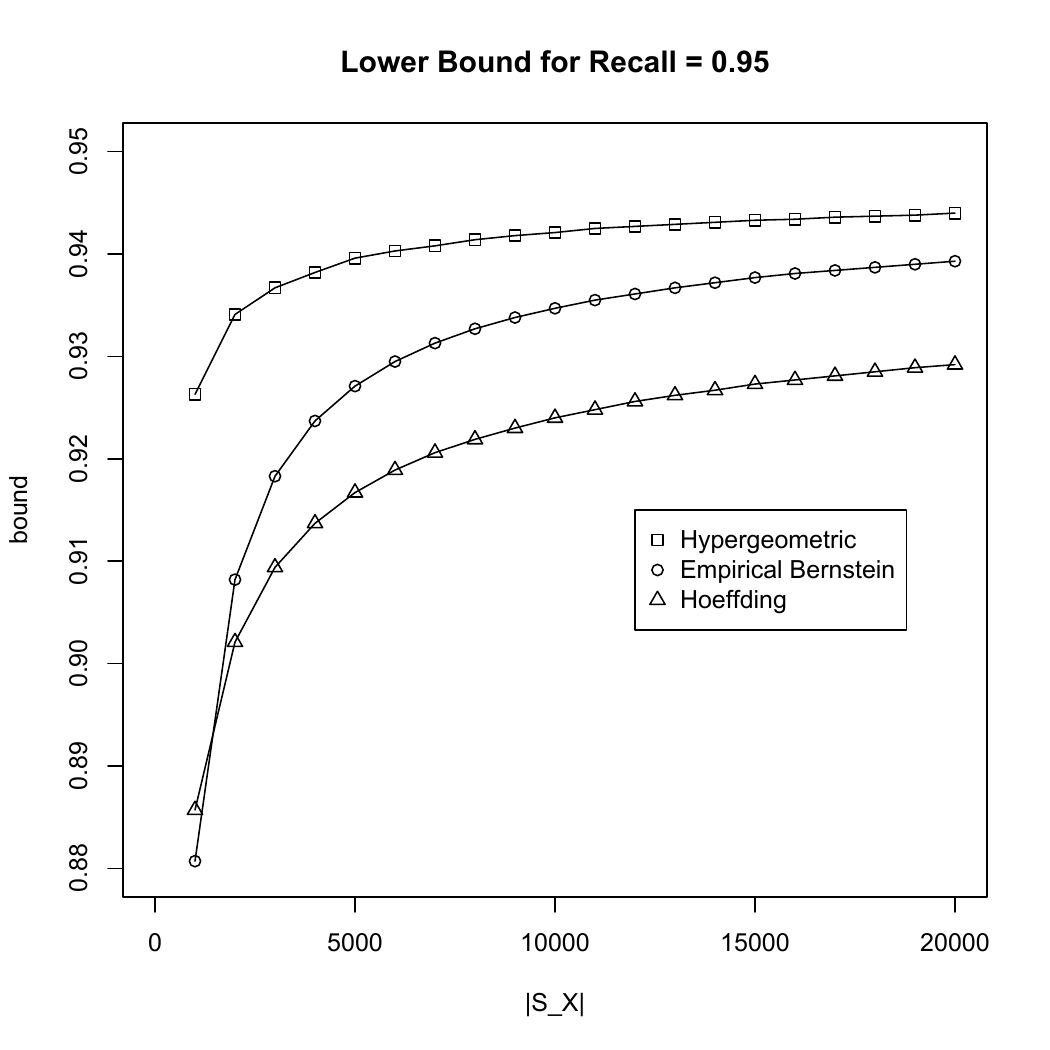}
  \caption{Lower bounds for complete recall using Theorem \ref{thm:complete_recall}.}
  \label{recall_95}
\end{figure}

Figure \ref{recall_85} shows lower bounds for complete recall using Theorem \ref{thm:complete_recall} if the holdout matching algorithm has 85\% recall ($r_H = 0.85$). Using hypergeometric bounds for $p^+()$ and $p^-()$, the lower bound is above 81\% even using $|S_X| = 1000$: a thousand samples from $X$ with verified matches. Hypergeometric bounds outperform the others for all numbers of samples, since they are sharp bounds (up to machine rounding errors and any error from using Stirling's approximation for factorials.) Empirical Bernstein bounds are stronger than Hoeffding bounds for 5000 or more verified matches. 

Figure \ref{recall_95} shows lower bounds for complete recall if the holdout matching algorithm has 95\% recall. With a more accurate matching algorithm, empirical Bernstein bounds begin to outperform Hoeffding bounds even for 2000 verified matches. Even with only 1000 samples with verified matches, hypergeometric bounds produce a lower bound on complete recall of over 92\%. 

\begin{figure}
  \includegraphics[width=\linewidth]{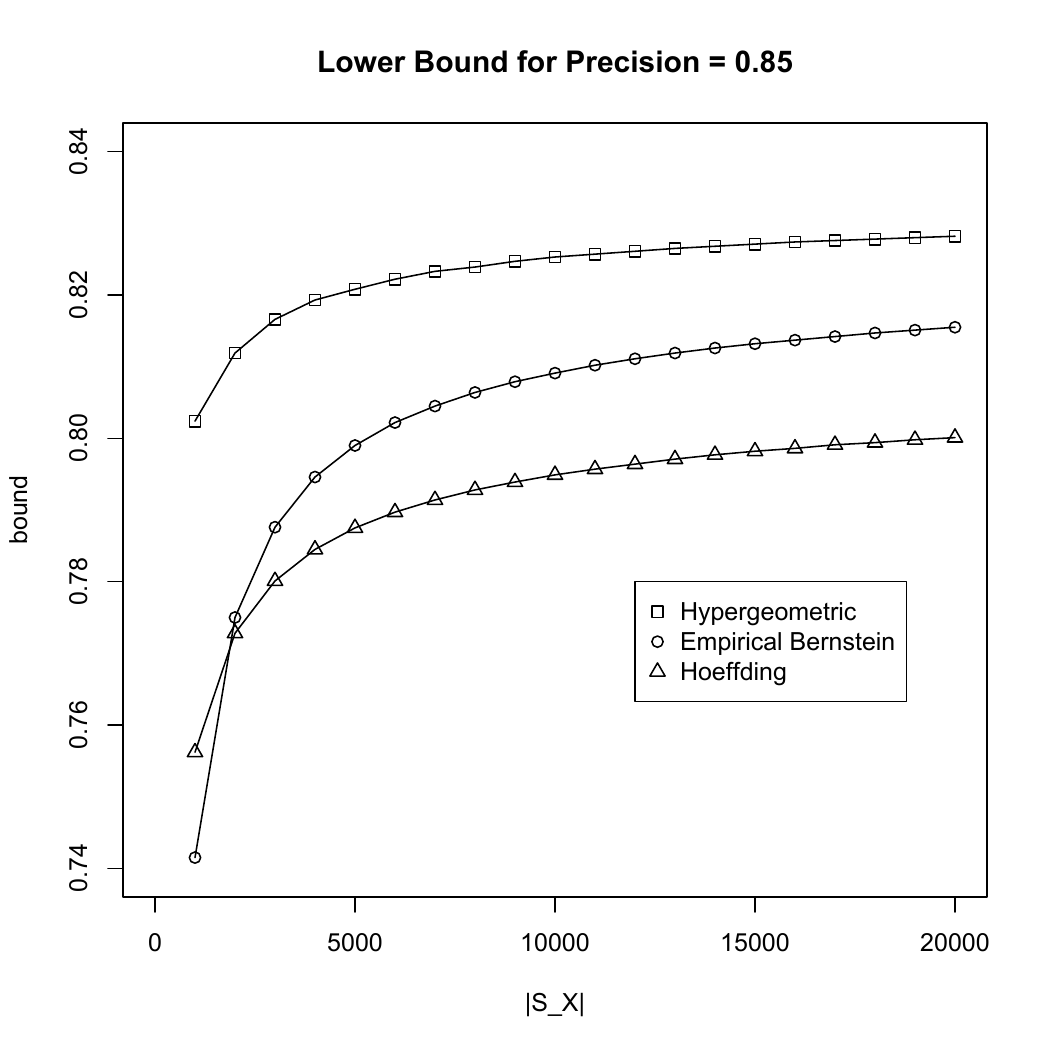}
  \caption{Lower bounds for complete precision using Theorem \ref{thm:complete_precision}.}
  \label{precision_85}
\end{figure}

\begin{figure}
  \includegraphics[width=\linewidth]{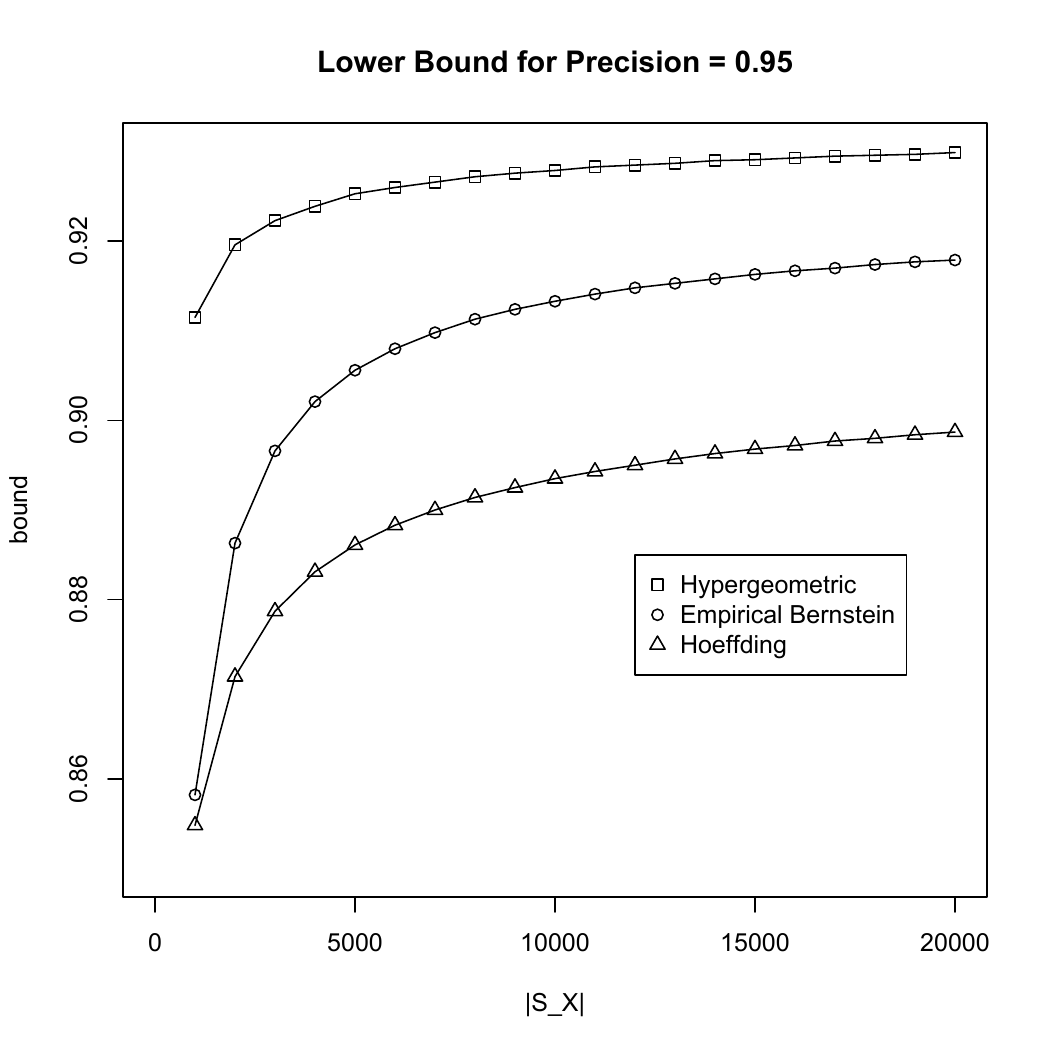}
  \caption{Lower bounds for complete precision using Theorem \ref{thm:complete_precision}.}
  \label{precision_95}
\end{figure}

Figures \ref{precision_85} and \ref{precision_95} show lower bounds for complete precision based on Theorem \ref{thm:complete_precision}. We use the same parameter values as for the recall results in Figures \ref{recall_85} and \ref{recall_95}, except $\delta = \frac{0.05}{4}$, to keep a five percent bound failure probability. We assume there is at most one holdout algorithm match and at most one complete algorithm match for each $x \in X$. We assume both algorithms match $\frac{2}{3}$ of the elements of $X$, and $d_p = 0.003$. The holdout precision rate $p_H$ is 0.85 for Figure \ref{precision_85} and 0.95 for Figure \ref{precision_95}. 

\section{Discussion} \label{sec:conclusion}
This paper introduces methods to validate precision and recall rates for network reconciliation algorithms. The methods apply regardless of the process that generated the networks' link structure and node data. In fact, the methods apply to matching algorithms in general, not just those that match nodes between networks. The methods include validation of matching algorithms that use all available data, so data used for validation may be used for matching as well. 

One direction for future research is to use the validation methods developed in this paper to guide network reconciliation processes. For example, some iterative matching algorithms begin with a seed set of matches and use matches from previous iterations to infer new matches in successive iterations, including methods from \cite{nunes12}, \cite{korula14}, \cite{pedarsani13}, and \cite{yartseva13}. In each iteration, validated precision and recall rates may be useful to determine which pairs to adopt as new matches and how much to rely on them to infer future matches. Pedarsani et al. \cite{pedarsani13} use Bayesian confidence measures in this way, under the assumption that the features used for matching have independent distributions. The validation methods developed in this paper could be used in a similar way when the feature distributions are correlated or unknown. 

Finally, it would be interesting to develop validation methods for network reconciliation in non-transductive settings, where the goal is to assess precision and recall for nodes to be added to the network in the future, whose data and links are not yet known. For these settings, we may need to assume that we know how the network will grow, or at least have a probabilistic model for network growth. Alternatively, we may be able to develop validation methods based on the assumption that soon-to-be-added nodes will be generated by the same distribution as recently-added nodes, even if we do not know the distribution, as in work on algorithm validation using cohorts \cite{bax13}. 

%\bibliographystyle{abbrv}
%\bibliography{bax}

\bibliographystyle{IEEEtran}
\bibliography{IEEEabrv,bax}

\section*{Appendix}

\subsection*{Proof of Subsampling Procedure from Section \ref{sec:sampling}}
\begin{IEEEproof}
Let $|X|=n$, $|D|=t$ and $|S|=s$.
Note that $|L|=t+s$. It suffices to show that 
\begin{eqnarray}
\prob{D} & = & \frac{1}{{n \choose t}},\\
\prob{S\vert D} & = & \frac{1}{{n \choose s}}.
\end{eqnarray}

Since $L$ is selected uniformly at random without replacement
from $X$, we have 
\begin{eqnarray}
\prob{D} & = & \sum_{L}\prob{L}\prob{D\vert L}\\
 & = & \sum_{L\supseteq D}\frac{1}{{n \choose t+s}{t+s \choose t}}\\
 & = & \frac{{n-t \choose s}}{{n \choose t+s}{t+s \choose t}}\\
 & = & \frac{1}{{n \choose t}}.
\end{eqnarray}
This shows that $D$ has the same distribution as a size-$t$
sample selected uniformly at random without replacement from $X$.

The distribution of $S$ conditional on $D$ satisfies

\be
\prob{S\vert D} = \sum_{L}\prob{L\vert D}\prob{S\vert D,L}
\ee
\be
= \sum_{L\supseteq D \cup S}\prob{L\vert D}\prob{i=|S\cap D|}\prob{S\vert N,D,L}
\ee
\be
= \frac{{n-t-s+|S\cap D| \choose |S\cap D|}{t \choose |S\cap D|}{n-t \choose s-|S\cap D|}}{{n-t \choose s}{n \choose s}{t \choose |S\cap D|}{s \choose s-|S\cap D|}}
\ee
\be
\frac{1}{{n \choose s}}.
\ee
Thus $S$ is independent of $D$ and has the same
distribution as a size-$s$ sample selected uniformly at random without
replacement from $X$.
\end{IEEEproof}

\end{document}